\documentclass[preprint,12pt]{elsarticle}

\usepackage[T1]{fontenc}
\usepackage[utf8]{inputenc}
\usepackage{lmodern}
\usepackage{microtype}
\usepackage{setspace}
\usepackage{amsmath,amssymb,amsthm,mathtools}
\usepackage{booktabs,threeparttable}
\usepackage{graphicx}
\usepackage[dvipsnames,table]{xcolor}
\usepackage{float}
\usepackage{placeins}
\usepackage{caption}
\usepackage{enumitem}
\usepackage{lineno}
\usepackage{tikz}
\usetikzlibrary{arrows.meta,positioning,shapes.geometric,fit,calc}
\usepackage[hidelinks]{hyperref}
\usepackage[nameinlink,noabbrev]{cleveref}
\biboptions{sort&compress}

\hypersetup{
	pdftitle={ARISE: An adaptive residual-informed stability ensemble for feature selection in small-sample biomedical omics},
	pdfauthor={Zardad Khan; Amjad Ali; Naz Gul; Sheema Gul; Saeed Aldahmani},
	pdfkeywords={feature selection; biomarker discovery; multiclass classification; stability; nested cross-validation; omics; medical artificial intelligence}
}

\definecolor{cgblue}{HTML}{0072B2}
\definecolor{cgsky}{HTML}{56B4E9}
\definecolor{cggreen}{HTML}{009E73}
\definecolor{cgorange}{HTML}{E69F00}
\definecolor{cgred}{HTML}{D55E00}
\definecolor{cgpink}{HTML}{CC79A7}
\definecolor{cgdark}{HTML}{1B263B}
\definecolor{cglight}{HTML}{EDF2F4}

\newcommand{\ARISE}{ARISE}
\newcommand{\R}{\mathbb{R}}

\newcommand{\argmax}{\operatorname*{arg\,max}}
\newcommand{\E}{\mathbb{E}}
\newcommand{\Prob}{\mathbb{P}}
\newcommand{\best}[1]{\textbf{#1}}
\newcommand{\second}[1]{\underline{#1}}
\newcommand{\DeltaSeven}{\Delta_{7}}

\newtheorem{lemma}{Lemma}
\newtheorem{proposition}{Proposition}

\theoremstyle{remark}

\newfloat{algorithm}{tbp}{loa}
\floatname{algorithm}{Algorithm}

\newcommand{\plotplaceholder}[5][!t]{%
	\begin{figure}[#1]
		\centering
		\IfFileExists{#2}{%
			\includegraphics[width=#3\linewidth]{#2}%
		}{%
			\fbox{\begin{minipage}[c][4.6cm][c]{0.91\linewidth}
					\centering\bfseries\color{cgdark}
					Figure file not found\\[0.6em]
					\normalfont\ttfamily\detokenize{#2}
			\end{minipage}}%
		}
		\caption{#4}
		\label{#5}
	\end{figure}%
}

\begin{document}
	\begin{frontmatter}
		
		\title{ARISE: An adaptive residual-informed stability ensemble for feature selection in small-sample biomedical omics}
		
		\author[uaeu]{Zardad Khan\corref{cor1}}
		
		\ead{zaar@uaeu.ac.ae}
		\author[uaeu]{Amjad Ali}
		\author[awkum]{Naz Gul}
		\author[uaeu]{Saeed Aldahmani\corref{cor1}} 		\ead{saldahmani@uaeu.ac.ae}
		\cortext[cor1]{Corresponding author.}

		\address[uaeu]{Department of Statistics and Business Analytics, United Arab Emirates University, Al Ain, United Arab Emirates}
		\address[awkum]{Department of Statistics, Abdul Wali Khan University Mardan, Mardan, Pakistan}
		
		\begin{abstract}
			\textbf{Objective:} Molecular classification involving a small number of observations requires feature-selection methods that identify predictive, stable, and nonredundant feature subsets while remaining applicable to both binary and multiclass outcomes. We propose \ARISE\ (Adaptive Residual-Informed Stability Ensemble), an adaptive feature-selection framework that integrates multiple complementary relevance signals, class-balanced stability assessment, residual-informed redundancy control, and multiclass pairwise coverage.
			
			\textbf{Methods:} The proposed \ARISE\ method combines seven percentile-normalized relevance components through fifteen predefined selector profiles whose contributions are adaptively weighted using nested inner cross-validation. The framework was evaluated on five molecular datasets using eight feature-set sizes, three fixed classifiers (k-nearest neighbours, support vector machine, and random forest), and six standalone filter-based comparators. Generalization performance was estimated using five-fold outer cross-validation repeated 50 times and assessed using balanced accuracy, macro-F1, and Cohen's $\kappa$.
			
			\textbf{Results:} Across a total of 210,000 held-out assessments on the five datasets, the prposed \ARISE\ method achieved the highest overall performance for all three evaluation metrics, corresponding to all 15 dataset--metric combinations. Its equal-dataset mean performance was 0.793 for balanced accuracy, 0.776 for macro-F1, and 0.725 for Cohen's $\kappa$, exceeding the strongest aggregate comparator by 0.022, 0.023, and 0.028, respectively. \ARISE\ maintained strong performance across a range of compact feature-set sizes, while the best-performing feature budget varied among datasets, indicating that no single subset size was uniformly optimal.
			
			\textbf{Conclusion:} \ARISE\ yields a transparent and adaptive framework for molecular feature selection that jointly addresses relevance, stability, redundancy, and multiclass discrimination. Its consistent performance across the five molecular datasets, multiple classifiers, evaluation metrics, and feature-set sizes supports its potential for small-sample molecular classification.
		\end{abstract}
		
		\begin{keyword}
			 medical artificial intelligence \sep feature selection  \sep multiclass classification \sep stability \sep nested cross-validation \sep omics \sep biomarker discovery
		\end{keyword}
		
	\end{frontmatter}
	
	\modulolinenumbers[5]
	
	\section{Introduction}
	\label{sec:introduction}
	
	Molecular diagnostic studies often involve hundreds to tens of thousands of candidate features measured on comparatively small cohorts comprising only tens or hundreds of observations. In this ``large $p$, small $n$'' regime, a selector must do more than maximize apparent discrimination: it should resist sampling perturbations, avoid redundant panels, preserve minority-class signals, and separate every tuning decision from the held-out assessment. These requirements are especially important in medical artificial intelligence (AI), where an unstable or leakage-contaminated panel may appear accurate yet fail mechanistic interpretation, assay transfer, or external validation.
	
	Conventional filters provide speed and transparent rankings, but each emphasizes a particular signal geometry. Mean-shift scores favor approximately location-separated classes; mutual information (MI) detects broader dependence; and interval overlap and Bhattacharyya separation capture distributional structure. No one geometry is uniformly appropriate across tissues, diseases, platforms, class imbalance, and feature budgets. Conversely, wrappers that search many subsets can be prohibitively variable in small samples and can overfit the model-selection criterion.
	
	Therefore, we developed \ARISE, \emph{Adaptive Residual-Informed Stability Ensemble}, a multiclass, stability-aware selector that (i) expresses complementary filters on a common percentile scale, (ii) estimates class-balanced subsample robustness, (iii) penalizes within-class residual correlation rather than raw correlation, (iv) rewards coverage of unresolved class pairs, and (v) tunes a fixed profile library with nested minimax-regret racing and soft consensus. Learner hyperparameters remain fixed across selectors so that the comparison targets feature selection rather than classifier optimization.
	
	The current manuscript offers four contributions. First, it gives a unified construction for binary and multiclass feature selection with exact nested feature paths. Second, it combines complementary relevance evidence with class-balanced subsample robustness, within-class residual redundancy control, and explicit multiclass pair coverage. Third, it establishes boundedness, pair-coverage submodularity, class-shift invariance of the redundancy term, conditional subsample-mean concentration, consensus feasibility, deterministic nesting, and outer-test separation. Fourth, it evaluates the proposed selector under fully nested repeated cross-validation across five heterogeneous molecular datasets, three fixed classifiers, and three performance metrics: balanced accuracy, macro-F1, and Cohen's $\kappa$.
	
	The study is intended as a methodological evaluation rather than evidence of clinical deployment. The datasets span different molecular platforms, dimensionalities, and multiclass structures, while the classifiers are deliberately held fixed to isolate the effect of feature selection. The rest of the paper is arranged as follows. Section~\ref{sec:related} reviews related work on feature selection, stability, ensemble approaches, and biomedical applications. Section~\ref{sec:methods} describes the datasets, the proposed \ARISE\ framework, its relevance, stability, redundancy, and pair-coverage components, the nested profile-racing and consensus procedure, the comparison methods, classifiers, and validation protocol. Section~\ref{sec:theory} presents the main theoretical properties of the proposed framework. Section~\ref{sec:results} reports the experimental results, including aggregate performance, dataset-wise comparisons, feature-budget behaviour, and selected-set stability. Section~\ref{sec:discussion} discusses the methodological implications, limitations, and directions for further research. Finally, the Conclusion summarizes the principal findings and perspectives of the study.

	\section{Related work}
	\label{sec:related}
	
	There are three common categories of feature-selection methods, i.e., filters, wrappers, and embedded approaches \cite{Saeys2007Bioinformatics}. Maximum-relevance/minimum-redundancy formalized a central filter trade-off \cite{Peng2005MRMR}; support-vector recursive elimination provided an influential wrapper for cancer expression data \cite{Guyon2002SVMRFE}; and the lasso and elastic net established sparse embedded estimation \cite{Tibshirani1996Lasso,Zou2005ElasticNet}. These methods remain useful, but fixed assumptions about effect geometry or correlation can be brittle across heterogeneous molecular studies.
	
	Chance-corrected stability and stability selection measures made resampling consistency an explicit objective \cite{Meinshausen2010Stability,Nogueira2018Stability}. Equally important, model selection and preprocessing must be nested: nonnested tuning creates optimistic bias \cite{Cawley2010SelectionBias}, a risk recognized early in microarray prognosis studies \cite{Simon2003Pitfalls}. \ARISE\ integrates these ideas, but uses subsampling stability as a score component rather than claiming the family-wise error control of classical stability selection.
	
	Recent research work has focusd on multi-objective and ensemble selection. Dual-stage bias correction, filter--wrapper stacking, rough-hypervolume search, low-order information criteria, bi-level ensembles, rank-revealing subspaces, and expert-augmented selection target complementary combinations of accuracy, stability, compactness, and interpretability \cite{Cattelani2025DOSA,Budhraja2023FWSE,Zhou2025Rough,Gao2023LowOrder,Zhou2023BiLevel,Moslemi2024RRQR,Rabiee2024Expert}. Work within medical AI and biomedical decision support has emphasized stable graph-based biomarker selection \cite{Chereda2024Stable}, multi-objective clinical feature search \cite{Zhou2025Rough}, and quantifiable interpretation of molecular response models \cite{Shi2025DRExplainer}. \ARISE\ differs by using a fixed interpretable component library, within-class residual redundancy, explicit class-pair coverage, and a nested soft consensus rather than selecting a single winning profile.
	
Feature discovery in biomedical regimes increasingly spans network, single-cell, multi-view, and multi-omics learning. Examples include directed regulatory graphs \cite{Wang2024CeRVE}, trajectory-preserving single-cell selection \cite{Ranek2024DELVE}, differentiable information imbalance \cite{Wild2025DII}, electronic-health-record-assisted omics \cite{Mataraso2025COMET}, incomplete multi-omics integration \cite{Ma2025IntegrateAnyOmics}, and cross-platform molecular prediction \cite{Wang2022CPOP}. Foundation models, graph methods, and multiscale association frameworks broaden the representational space \cite{Pai2024Foundation,Valous2024GraphOmics,Schweickart2024AutoFocus}, while recent reviews emphasize precision immunotherapy and knowledge-guided genomics \cite{Li2024Immunotherapy,Guo2023BiomedicalGenomics}. Graphical ensemble selection additionally models dependencies among co-selected medical features \cite{Battistella2024Graphical}. Explainable gene-biomarker models and reviews stress that attribution alone is not clinical validity \cite{Yagin2023XAI,Ali2023Enlightening}; early-integration and training-dynamics studies similarly motivate explicit robustness audits \cite{Zhang2024MOINER,Karin2024Annotatability}. Our narrower objective is a transparent selector that can precede several fixed learners and remains computationally feasible when $p\gg n$.
	
	\section{Materials and methods}
	\label{sec:methods}
	
	\subsection{Study datasets and data provenance}
	\label{sec:data}
	\label{sec:provenance}
	
	The current evaluation utilized the five molecular datasets summarized in Table \ref{tab:datasets}. They span methylation and gene-expression measurements, three to six analysis classes, and feature dimensions from 45 to 22,283. Each dataset was evaluated with the same three fixed classifiers, i.e., kNN, SVM, and random forest, and the same three reported metrics, i.e.,balanced accuracy, macro-F1, and Cohen's $\kappa$. The dataset is the unit of across-study aggregation; folds, feature budgets, classifiers, and predictions are treated as repeated measurements within a dataset rather than as independent studies.
	
	\begin{table}[h]
		\centering
		\begin{threeparttable}
			\caption{Benckmark molecular datasets used in the benchmark evaluation. Here, $n$, $p$, and $K$ denote sample size, number of features, and number of classes, respectively. }
			\label{tab:datasets}
			\small
			\setlength{\tabcolsep}{20.0pt}
			\begin{tabular}{@{}clrrrl@{}}
				\toprule
				ID & Dataset & $n$ & $p$ & $K$ & Source \\
				\midrule
				D1 & HCC methylation & 56 & 45 & 3 & \cite{Archer2010HCC} \\
				D2 & Nutrimouse & 40 & 120 & 5 & \cite{Martin2007Nutrimouse} \\
				D3 & Yeoh ALL & 248 & 12625 & 6 & \cite{Yeoh2002ALL} \\
				D4 & Su & 102 & 5565 & 4 & \cite{Su2002Transcriptomes,Hochreiter2010FABIA} \\
				D5 & Burczynski & 127 & 22283 & 3 & \cite{Burczynski2006IBD} \\
				\bottomrule
			\end{tabular}
			
		\end{threeparttable}
	\end{table}
	
	The benchmark datasets were obtained from established R data resources. D1 used \texttt{hccframe} from \texttt{ordinalgmifs} version 1.0.9; D2 used \texttt{nutrimouse} from \texttt{whitening} version 1.4.0; and D3--D5 used \texttt{yeoh}, \texttt{su}, and \texttt{burczynski}, respectively, from \texttt{datamicroarray} version 0.2.3. Object-level provenance is reported because the analyzed dimensions can differ from those in the originating studies.
	
	D1 is a 45-CpG subset of GSE18081 for which technical replicates and matched cirrhotic specimens belonging to hepatocellular-carcinoma subjects were removed \cite{Archer2010HCC,Archer2014Ordinalgmifs}. D2 uses diet as the classification endpoint, while genotype is balanced but not modeled \cite{Martin2007Nutrimouse}. D4 preserves the processed object's native endpoint codes; no clinical tissue names are inferred. Its provenance links the original transcriptome resource to the later FABIA data layer \cite{Su2002Transcriptomes,Hochreiter2010FABIA}. D5 traces to GSE3365. These dataset-specific characteristics constrain biological interpretation. All additional preprocessing, feature scoring, tuning, and model fitting introduced in the present study were performed within the applicable training partition.
	
	\subsection{Notation and component construction}
	
	Let $\mathcal D=\{(\mathbf x_i,y_i)\}_{i=1}^{n}$, where $\mathbf x_i\in\R^p$ and $y_i\in\{1,\ldots,K\}$. Let $M=K(K-1)/2$ denote the number of unordered class pairs and let $J=7$ denote the number of relevance components. Every quantity estimated by the study pipeline, including additional preprocessing, component scores, subsampling quantities, profile tuning, feature paths, and classifier parameters, is learned from the applicable training partition. Transformations already embedded in the published source objects are treated as fixed inputs and are not re-estimated.
	
	For feature $g$ and pair $(a,b)$, let $\mu_{ag},s^2_{ag},n_a$ denote the class mean, variance, and size; set $v_{ag}=\max(s^2_{ag},\epsilon)$, $\epsilon=10^{-12}$, and $\delta_{abg}=\mu_{bg}-\mu_{ag}$. Six pairwise components are
	\begin{align}
		D^{\mathrm F}_{abg}&=\frac{\delta_{abg}^2}{v_{ag}+v_{bg}+\epsilon}, &
		D^{\mathrm{SNR}}_{abg}&=\frac{|\delta_{abg}|}{\sqrt{v_{ag}}+\sqrt{v_{bg}}+\epsilon},\\
		D^{\mathrm W}_{abg}&=\frac{|\delta_{abg}|}{\sqrt{v_{ag}/n_a+v_{bg}/n_b+\epsilon}}, &
		D^{\mathrm B}_{abg}&=\frac14\log\!\left[\frac14\left(\frac{v_{ag}}{v_{bg}}+\frac{v_{bg}}{v_{ag}}+2\right)\right]
		+\frac{\delta_{abg}^2}{4(v_{ag}+v_{bg})}.
	\end{align}
	For POS, the proposed cache uses whisker $\omega=1.5$ and robust class interval
	\[
	I_{ag}=[L_{ag},U_{ag}]=\left[\max\{\min x_{ag},Q_{1,ag}-\omega\,\mathrm{IQR}_{ag}\},
	\min\{\max x_{ag},Q_{3,ag}+\omega\,\mathrm{IQR}_{ag}\}\right].
	\]
	Let $O_{abg}=\max\{0,\min(U_{ag},U_{bg})-\max(L_{ag},L_{bg})\}$ and $H_{abg}=\max(U_{ag},U_{bg})-\min(L_{ag},L_{bg})$. The POS score is $D^{\mathrm{POS}}_{abg}=1-O_{abg}/H_{abg}$ when $H_{abg}>\epsilon$ and zero otherwise. Rank-based pair separation is $D^{\mathrm A}_{abg}=|2\mathrm{AUC}_{abg}-1|$. Here, the pairwise AUC is used only to construct this feature-relevance component; it is not an evaluation metric. Each component is converted to a feature percentile within its class pair and then averaged equally over pairs. MI uses three quantile bins by default and is also percentile-normalized. Thus every feature has $\mathbf p_g\in[0,1]^7$.
	
	\subsection{Profile relevance, stability, redundancy, and pair coverage}
	
	Each fixed candidate profile $c\in\{1,\ldots,C\}$, with $C=15$, has nonnegative component weights $\mathbf w_c\in\DeltaSeven$, where
	\[
	\DeltaSeven=\left\{\mathbf w\in\R_+^7:\sum_{j=1}^{7}w_j=1\right\}.
	\]
	Its relevance is
	\[
	R_{gc}=\mathbf w_c^\top\mathbf p_g.
	\]
	Let $n_{\min}=\min_k n_k$ and $m=\min\{n_{\min},\max(2,\lfloor0.70n_{\min}\rfloor)\}$. For each of $B=8$ independent draws, $m$ observations are sampled without replacement from every class. The POS, Fisher, SNR, Welch, Bhattacharyya, rank-based, and MI components are recomputed. The resulting profile-weighted scores are re-percentiled across the screened features to obtain $Z^{(b)}_{gc}$. The subsample-averaged robustness score and combined base score are
	\[
	S_{gc}=\frac1B\sum_{b=1}^{B}Z^{(b)}_{gc},\qquad
	A_{gc}=(1-\alpha_c)R_{gc}+\alpha_cS_{gc},\qquad \alpha_c\in[0,1].
	\]
	The quantity $S_{gc}$ is a mean subsample rank rather than a dispersion measure or selection frequency; selected-set Jaccard stability is evaluated separately.
	
	To avoid treating replicated class signal as nuisance correlation, redundancy is calculated after within-class centering. Let
	\[
	z_{ig}=\frac{x_{ig}-\overline{x}_{y_i g}}{s^{\mathrm{within}}_g},\qquad
	\rho_g(\mathcal S)=\max_{h\in\mathcal S}
	\left|\frac{\mathbf z_g^\top\mathbf z_h}{n-K}\right|,
	\]
	where $s^{\mathrm{within}}_g=\{\sum_i(x_{ig}-\overline{x}_{y_i g})^2/(n-K)\}^{1/2}$. A nonfinite or $<10^{-12}$ within-class scale is replaced by one, and any remaining nonfinite residual is set to zero. We set $\rho_g(\varnothing)=0$. No hard correlation deletion or mandatory cluster representative is used.
	
	For multiclass coverage, let $U^{(c)}_{gm}\in[0,1]$ be the profile-$c$ weighted pairwise utility for pair $m$, using the six pairwise components and their weights renormalized to sum to one. When $K=2$, or when profile $c$ has zero total weight on those six components (the pure MI profile), we set $F_c\equiv0$ and $\eta_c=0$. Otherwise define
	\[
	F_c(\mathcal S)=\frac1M\sum_{m=1}^{M}\max_{g\in\mathcal S}U^{(c)}_{gm},
	\qquad F_c(\varnothing)=0.
	\]
	Let $\mathcal G_c$ be profile $c$'s deterministic local candidate universe after screening, with $|\mathcal G_c|\ge q_{\max}$. Let $\eta_c\ge0$ denote the pair-coverage weight, $\lambda_{0c}\ge0$ the initial redundancy penalty, and $d_c\in[0,1]$ its linear decay parameter. At path step $t\le q_{\max}$, \ARISE\ selects
	\begin{equation}
		g_t=\argmax_{g\in\mathcal G_c\setminus\mathcal S_{t-1}}
		\left[A_{gc}+\eta_c\Delta F_c(g\mid\mathcal S_{t-1})
		-\lambda_c(t)\rho_g(\mathcal S_{t-1})\right],
		\label{eq:greedy}
	\end{equation}
	where $\Delta F_c(g\mid\mathcal S)=F_c(\mathcal S\cup\{g\})-F_c(\mathcal S)$ and
	\[
	\lambda_c(t)=
	\begin{cases}
		\lambda_{0c}\left[1-d_c\dfrac{t-1}{q_{\max}-1}\right], & q_{\max}>1,\\[0.4em]
		\lambda_{0c}, & q_{\max}=1.
	\end{cases}
	\]
	The output is one ordered path; all feature budgets are exact prefixes.
	
	\subsection{Two-stage nested profile racing and soft consensus}
	
	The 15 profiles comprise eight mixed profiles (balanced, location-shift, distributional, nonlinear, stability-first, diversity-first, low-redundancy, and maximin) and seven pure components. Within each outer-training partition, three-fold inner cross-validation performs two-stage racing. Stage~1 evaluates all profiles at budgets 5 and 15 with fixed kNN and SVM. A one-standard-error rule first defines the preferred set; exactly three profiles proceed to Stage~2, padding with the next-lowest-loss profiles when fewer than three qualify and truncating by deterministic loss order when more than three qualify. Stage~2 evaluates those three profiles at all eight budgets and all three classifiers.
	
	The three performance metrics used for profile racing and for reported evaluation are
	\[
	\mathcal L_{\mathrm T}=\mathcal L_{\mathrm R}
	=\{\mathrm{BA},\mathrm{MacroF1},\kappa\}.
	\]
	Let $m_{cf\ell}\in[0,1]$ be the inner-fold utility for metric $\ell\in\mathcal L_{\mathrm T}$ and configuration $c$, after mapping $\kappa$ to $(\kappa+1)/2$. Within fold $f$, $m_{cf\ell}$ is the equal mean over the applicable stage conditions: four cells in Stage~1 (two budgets $\times$ two classifiers) and 24 cells in Stage~2 (eight budgets $\times$ three classifiers). The fold-specific regret and loss are
	\begin{equation}
		r_{cf\ell}=\max_{c'}m_{c'f\ell}-m_{cf\ell},\qquad
		L_{cf}=0.70\max_{\ell\in\mathcal L_{\mathrm T}} r_{cf\ell}
		+0.30\,|\mathcal L_{\mathrm T}|^{-1}\sum_{\ell\in\mathcal L_{\mathrm T}}r_{cf\ell}.
		\label{eq:regret}
	\end{equation}
	After Stage~2, the best-plus-one-SE rule is recomputed among the three survivors; at least the best profile is retained in $\mathcal E$. Let $\overline L_c$ and $\mathrm{SE}_c$ be the mean and standard error of fold losses. Eligible profiles receive soft weights
	\begin{equation}
		\omega_c=\frac{\exp[-(\overline L_c-\overline L_{\min})/\tau]/(1+\mathrm{SE}_c/\tau)}
		{\sum_{j\in\mathcal E}\exp[-(\overline L_j-\overline L_{\min})/\tau]/(1+\mathrm{SE}_j/\tau)},
		\qquad \tau=0.02,
		\label{eq:soft}
	\end{equation}
	and every continuous path hyperparameter is averaged as $\theta^*=\sum_{c\in\mathcal E}\omega_c\theta_c$. Component scores, stability, and the greedy path are then recomputed on the full outer-training partition using this consensus.
	\begin{algorithm}[h]
		\caption{ARISE training and nested feature-path construction.}
		\label{alg:arise}
		\scriptsize
		\hrule
		\vspace{0.45em}
	
		\textbf{Input:} training data $\mathcal T$, budgets $\mathcal Q=\{1,5,10,15,20,25,30,35\}$, 15 fixed profiles, $B=8$, and seed.\\
		\textbf{Output:} one ordered consensus path whose prefixes give the selected sets for all $q\in\mathcal Q$.
		
		\begin{enumerate}[label=\textbf{\arabic*.},leftmargin=2.2em,itemsep=0.13em,topsep=0.35em]
			\item Create three stratified inner folds within $\mathcal T$.
			\item For each fold $f$, fit imputation and variance filtering on its inner-training subset $\mathcal T_{-f}$ and apply the frozen operations to validation subset $\mathcal V_f$.
			\item From $\mathcal T_{-f}$ only, construct the seven percentile-normalized components and deterministic candidate workspace.
			\item For every profile $c$, compute $R_{gc}$; draw $B$ balanced subsamples to obtain $S_{gc}$; and form $A_{gc}$.
			\item On $\mathcal T_{-f}$, compute residual redundancy, enabled pair utilities, and the profile path by \cref{eq:greedy}.
			\item Freeze each prefix, standardize its columns using $\mathcal T_{-f}$, fit kNN and SVM at $q\in\{5,15\}$, and score $\mathcal V_f$.
			\item Aggregate Stage~1 losses across folds and advance exactly three profiles by the padded/capped one-standard-error rule.
			\item For each $f$, reuse its training-only workspace to evaluate those survivors over $\mathcal Q$ and all fixed learners on $\mathcal V_f$.
			\item Compute \cref{eq:regret}, the final one-standard-error set, and soft weights $\omega_c$ by \cref{eq:soft}.
			\item Refit imputation and variance filtering on all of $\mathcal T$ and construct its component cache and candidate workspace.
			\item Recompute consensus relevance, balanced-subsample robustness, residual redundancy, pair coverage, and the greedy path on all of $\mathcal T$.
			\item Return the ordered consensus path; each requested selected set is its exact prefix.
		\end{enumerate}
		
		\vspace{0.25em}
		\hrule

	\end{algorithm}

	\Cref{fig:flow} summarizes the proposed method. The main sequence follows component construction, profile-specific path generation, inner-cross-validation racing, and soft consensus before refitting on the complete training partition.
	
	\begin{figure}[H]
		\begin{singlespace}
			\centering
			\begin{tikzpicture}[
				node distance=4mm and 10mm,
				font=\scriptsize,
				block/.style={rectangle,rounded corners,draw=cgdark,thick,fill=cglight,
					align=center,minimum height=6mm,text width=44mm,inner sep=1.4mm},
				emph/.style={block,fill=cgsky!22,draw=cgblue},
				side/.style={rectangle,rounded corners,draw=cgdark,fill=white,
					align=center,minimum height=7mm,text width=28mm,inner sep=1.2mm},
				output/.style={block,fill=cggreen!16,draw=cggreen!70!black},
				arrow/.style={-{Latex[length=2mm]},thick,draw=cgdark}
				]
				
				```
				\node[block] (train) {Current training partition};
				\node[emph,below=of train] (cache) {Seven normalized relevance components\\and candidate screening};
				
				\node[side,below left=6mm and 10mm of cache] (stable)
				{Class-balanced subsample\\robustness $S_{gc}$};
				\node[side,below right=6mm and 10mm of cache] (relevance)
				{Profile-weighted\\relevance $R_{gc}$};
				
				\node[emph,below=20mm of cache] (base)
				{Combined score\\$A_{gc}=(1-\alpha_c)R_{gc}+\alpha_cS_{gc}$};
				
				\node[side,below left=6mm and 10mm of base] (resid)
				{Residual redundancy\\penalty};
				\node[side,below right=6mm and 10mm of base] (coverage)
				{Pairwise coverage\\gain};
				
				\node[emph,below=19mm of base] (paths)
				{Greedy nested feature paths};
				
				\node[block,below=of paths] (race)
				{Inner-CV profile race\\15 profiles $\rightarrow$ 3 survivors};
				
				\node[emph,below=of race] (consensus)
				{One-SE selection and\\soft consensus};
				
				\node[output,below=of consensus] (final)
				{Final consensus path\\for $q\in\mathcal Q$};
				
				\draw[arrow] (train)--(cache);
				\draw[arrow] (cache.south west)--(stable.north);
				\draw[arrow] (cache.south east)--(relevance.north);
				\draw[arrow] (stable.south)|-(base.west);
				\draw[arrow] (relevance.south)|-(base.east);
				\draw[arrow] (train.west)--++(-42mm,0)|-(resid.west);
				\draw[arrow] (train.east)--++(42mm,0)|-(coverage.east);
				\draw[arrow] (resid.south)|-(paths.west);
				\draw[arrow] (coverage.south)|-(paths.east);
				\draw[arrow] (base)--(paths);
				\draw[arrow] (paths)--(race);
				\draw[arrow] (race)--(consensus);
				\draw[arrow] (consensus)--(final);
				
			\end{tikzpicture}
			\caption{Architecture of the proposed ARISE feature-selection framework.}
			\label{fig:flow}
		\end{singlespace}
		
	\end{figure}
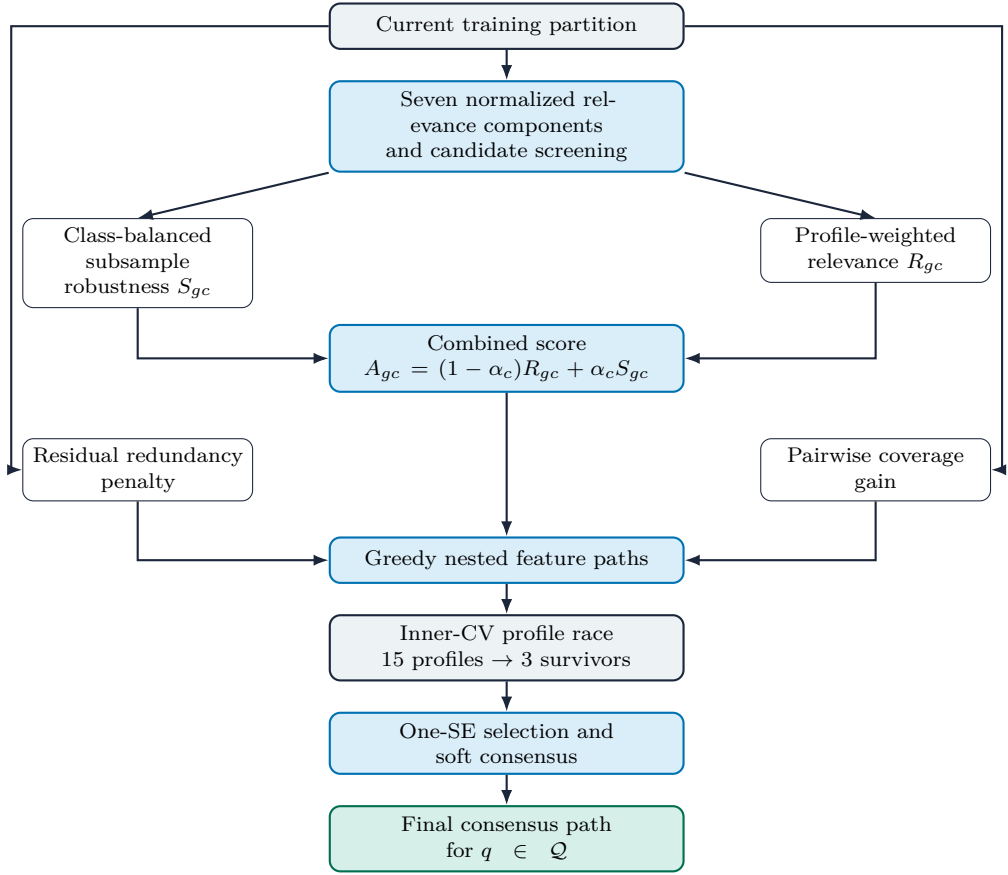
	
	\subsection{Comparators and fixed classifiers}
	
	The six standalone comparators were POS, Fisher score, signal-to-noise ratio (SNR), MI, Bhattacharyya distance, and Welch score. Small fixed grids were used only where a natural hyperparameter existed: POS interval whisker $\omega\in\{1,1.5,2\}$ and MI bins $\in\{2,3,4\}$. Within each method, the choice minimized the same 0.70 worst-regret plus 0.30 mean-regret objective over three inner folds, using fixed SVM and budget 15. For each parameter-free pairwise filter, every class-pair score was computed, features were percentile-ranked within that pair, and the percentiles were averaged equally over the $K(K-1)/2$ pairs.
	
	Three classifiers were used throughout the main evaluation, with hyperparameters fixed identically across selectors: distance-weighted kNN with $k=5$; radial SVM with $C=1$, $\gamma=1/q$, shrinking enabled, and probability fitting disabled; and random forest \cite{Breiman2001RandomForests} with 120 trees, $\mathrm{mtry}=\lfloor\sqrt q\rfloor$, terminal node size 1, and bootstrap replacement. In each current training partition, class $k$ received weight $n/(Kn_k)$. For kNN this multiplied inverse-distance neighbor votes, while SVM and random forest received their native class-weight arguments. For every classifier, the predicted label was the argmax of its class-score matrix. kNN supplied normalized weighted votes and random forest supplied normalized vote proportions. For SVM, each pairwise decision margin was oriented by its training-class mean margins, divided by the corresponding training-pair standard deviation, transformed by the logistic function, accumulated for each class, and divided by the number of class pairs. The design therefore estimates average selector utility across three fixed, heterogeneous decision rules rather than the maximum attainable accuracy of separately tuned classifiers.
	
	The fixed library comprised eight mixed profiles (balanced, location-shift, distributional, nonlinear, stability-first, diversity-first, low-redundancy, and maximin) and seven pure-component profiles. The seven pure profiles place unit weight on one relevance component at a time, whereas the mixed profiles combine components according to their named design objective. 
	
	\subsection{Nested validation, metrics, and aggregation}
	
	The outer design used five stratified folds repeated fifty times (250 held-out fold values per budget--classifier condition); the inner design used one three-fold split inside each outer training set. The same three classifiers, i.e., kNN, SVM, and random forest, were evaluated for every dataset, method, and feature budget. All newly fitted missing-value handling, variance filtering, scaling, scoring, screening, profile selection, and model fitting were training-only. This procedural separation cannot reverse filtering, imputation, normalization, or other transformations already present in an upstream package object. For class $k$, let $R_k$ and $P_k$ denote recall and precision. The three reported performance metrics were
	\begin{align}
		\mathrm{BA}&=K^{-1}\sum_{k=1}^{K}R_k,\\
		\mathrm{MacroF1}&=K^{-1}\sum_{k=1}^{K}\frac{2P_kR_k}{P_k+R_k},\\
		\kappa&=\frac{p_o-p_e}{1-p_e},
	\end{align}
	where undefined class F1 terms were assigned zero, $p_o$ is observed agreement, and $p_e$ is chance agreement from predicted and observed margins. Balanced accuracy gives each class equal recall weight, macro-F1 additionally requires classwise precision, and $\kappa$ corrects overall agreement for agreement expected from the marginal label distributions. Their joint use therefore distinguishes minority-class recognition, precision--recall balance, and chance-corrected concordance.
	
	For each dataset, method, and metric, we first averaged the 250 outer-fold values within every budget--classifier condition, then averaged the 24 condition means equally. Methods were ranked descending with average ranks for exact ties. Across-dataset means therefore give every dataset equal weight, irrespective of its sample size.
	
	\subsection{Exploratory statistical analysis}
	
	Across the five dataset-level means, we used the Friedman statistic as an omnibus rank diagnostic. Proposed-versus-comparator differences were summarized by exact two-sided sign tests; six comparator $P$ values were Holm-adjusted within each reported metric. Repeated folds were never treated as independent replicates. Because only five heterogeneous datasets were available for across-study inference, these tests are interpreted as exploratory diagnostics; emphasis is placed on effect sizes, dataset-level consistency, and distributional overlap rather than on confirmatory significance.
	
	\subsection{Computational details}
	
	The computational procedure caches univariate components and pairwise summaries within each training fold, reuses nested feature prefixes, and parallelizes outer folds. Each profile's \emph{Screen} value is a local candidate cap applied after descending base-score ordering (525, 630, or 700 for mixed profiles and 700 for pure profiles). The consensus cap is $\max\{100,\operatorname{round}(\sum_c\omega_c\,\mathrm{Screen}_c)\}$. A shared workspace is formed from the union of each profile's top-$q_{\max}$ relevance-ranked features and is filled deterministically by mean relevance to at most 700 features. Fixed random seeds, stored split assignments, deterministic tie handling, and completeness checks were used to support reproducible evaluation.
	
	\section{Theoretical properties}
	\label{sec:theory}
	
	The following results characterize the selector conditional on a training sample and its fixed candidate-profile library. They do not claim that a selected panel is biologically causal or universally optimal.
	
	\begin{lemma}[Bounded construction]
		\label{lem:bounded}
		If component and subsample percentiles lie in $[0,1]$, $\mathbf w_c\in\DeltaSeven$, and $\alpha_c\in[0,1]$, then $R_{gc}$, $S_{gc}$, $A_{gc}$, $F_c(\mathcal S)$, and $\Delta F_c(g\mid\mathcal S)$ all lie in $[0,1]$; when pair coverage is enabled, so does $U^{(c)}_{gm}$.
	\end{lemma}
	\begin{proof}
		The relevance, subsample-robustness mean, and pair utility are convex combinations of bounded quantities. A coordinatewise maximum of values in $[0,1]$ remains bounded, as does its nonnegative average. Because adding a feature cannot reduce any coordinatewise maximum, the marginal gain is nonnegative and at most one.
	\end{proof}
	
	\begin{proposition}[Monotone submodular pair coverage]
		\label{prop:submodular}
		$F_c$ is normalized, monotone, and submodular. In the enabled multiclass case, its implemented gain is
		\[
		\Delta F_c(g\mid\mathcal S)=\frac1M\sum_{m=1}^{M}
		\left(U^{(c)}_{gm}-\max_{h\in\mathcal S}U^{(c)}_{hm}\right)_+.
		\]
	\end{proposition}
	\begin{proof}
		For each pair $m$, $\max_{g\in\mathcal S}U^{(c)}_{gm}$ is zero at $\varnothing$, is nondecreasing with $\mathcal S$, and has diminishing returns because the current maximum can only increase. Nonnegative averaging over pairs preserves normalization, monotonicity, and submodularity. When coverage is disabled, $F_c\equiv0$ has the same properties trivially.
	\end{proof}
	
	\begin{lemma}[Within-class shift invariance]
		\label{lem:shift}
		For transformed values $x'_{ig}=x_{ig}+a_{y_i g}$, where $a_{kg}$ is any class-specific constant, the residualized vector $\mathbf z_g$ and every redundancy correlation are unchanged, provided the within-class scale is nonzero.
	\end{lemma}
	\begin{proof}
		The transformed class mean is $\overline{x}'_{kg}=\overline{x}_{kg}+a_{kg}$, so $x'_{ig}-\overline{x}'_{y_i g}=x_{ig}-\overline{x}_{y_i g}$. The within-class scale and all residual cross-products are consequently identical.
	\end{proof}
	
	\begin{proposition}[Deterministic nested path]
		\label{prop:path}
		Conditional on training data, configuration, random seed, a candidate universe containing at least $q_{\max}$ finite-score features, and the feature-universe order, the greedy routine returns one no-duplicate ordered path. If $q_1<q_2$, then $\mathcal S_{q_1}\subset\mathcal S_{q_2}$ and $\mathcal S_{q_1}$ is exactly the first $q_1$ entries of $\mathcal S_{q_2}$.
	\end{proposition}
	\begin{proof}
		Before greedy updates, local candidates are ordered by decreasing base score and then original feature index. At each step, selected features are masked; any subsequent score tie is resolved by that fixed local order, so exactly one unselected feature is appended. Induction over $t$ gives uniqueness and the prefix property.
	\end{proof}
	
	\begin{lemma}[Conditional subsample-mean concentration]
		\label{lem:hoeffding}
		Let $\mathcal H$ be the sigma-field generated by $\mathcal D$, the fixed screen and cache, and all nonsubsampling randomness. Conditional on $\mathcal H$, assume the balanced subsamples are drawn independently, and let $Z_{gc}^{(b)}\in[0,1]$ for the $p_s$ screened features. For $\varepsilon>0$,
		\[
		\Prob\left(\left|S_{gc}-\E[S_{gc}\mid\mathcal H]\right|\ge\varepsilon\mid\mathcal H\right)
		\le 2\exp(-2B\varepsilon^2).
		\]
		Simultaneously over $p_s$ screened features and $C$ profiles, the union bound is
		\[
		\Prob\left(
		\max_{\substack{1\le g\le p_s\\1\le c\le C}}
		\left|S_{gc}-\E[S_{gc}\mid\mathcal H]\right|\ge\varepsilon
		\,\middle|\,\mathcal H\right)
		\le \min\{1,2p_sC\exp(-2B\varepsilon^2)\}.
		\]
	\end{lemma}
	\begin{proof}
		Apply Hoeffding's inequality to the $B$ conditionally independent bounded variables and then the union bound. The result describes Monte Carlo concentration of a mean rank, not biological selection consistency or dispersion. With $B=8$ and as many as 700 screened features, the simultaneous union bound can still be numerically vacuous and is included only as a formal finite-sample statement.
	\end{proof}
	
	\begin{proposition}[Regret dominance]
		\label{prop:regret}
		After normalizing $\kappa$ to $[0,1]$, $r_{cf\ell},L_{cf}\in[0,1]$. If configuration $c$ weakly dominates $c'$ on every utility in $\mathcal L_{\mathrm T}$ within an inner fold, then $L_{cf}\le L_{c'f}$.
	\end{proposition}
	\begin{proof}
		The best metric value in the fold is common to both regret definitions, so dominance implies $r_{cf\ell}\le r_{c'f\ell}$ for every $\ell$. Both the maximum and arithmetic mean are monotone coordinatewise, proving the loss inequality.
	\end{proof}
	
	\begin{proposition}[Consensus feasibility]
		\label{prop:consensus}
		The weights in \cref{eq:soft} are nonnegative and sum to one. Hence $\theta^*$ lies in the convex hull of eligible profiles; in particular, the consensus component weights remain on the simplex.
	\end{proposition}
	\begin{proof}
		Each unnormalized factor is strictly positive and the denominator is their sum. Convex combinations remain in every convex feasible set shared by the eligible parameters.
	\end{proof}
	
	\begin{proposition}[Outer-test separation]
		\label{prop:leakage}
		Given fixed outer splits and a fixed input data object, every preprocessing statistic newly fitted by the evaluation pipeline, selector score, tuning decision, feature path, and classifier fit is measurable with respect to the outer-training data and algorithmic random seeds. Outer-test values are only transformed by the frozen training-fitted operations and supplied to the frozen model for prediction and metric calculation.
	\end{proposition}
	\begin{proof}
		By construction, transformations, selectors, and models are fitted exclusively on inner or full outer-training partitions. Feature alignment, mean imputation, selected-column centering and scaling, and the fitted classifier are then frozen before application to outer-test values. This establishes procedural test separation conditional on the input object; it does not reverse any preprocessing already embedded in that object before the study pipeline begins.
	\end{proof}
	
	\paragraph{Complexity.}
	Conditional on an already constructed full component cache and a fixed screen, computing $C$ profile relevances costs $O(pJC)$. For a screened universe $p_s$, a target path $q$, and $M$ class pairs, the incremental greedy path stage uses $O(qnp_s+qMp_s)$ operations per profile and $O(np_s+p_sC+p_sM)$ additional working memory. Here $J=7$ components, $C=15$, $p_s\le700$, and $q\le35$. These are not end-to-end bounds: total computation also includes the full input and component cache (including $p\times M$ pairwise arrays), feature ordering during screening, and rebuilding $B=8$ subsample ranks on the screened set.
	
	\section{Results}
	\label{sec:results}
	
	\subsection{Evaluation integrity and aggregate method performance}
	
The evaluation comprised 210,000 held-out assessments across five datasets, seven feature-selection methods, eight feature budgets, three classifiers, and 250 outer-fold evaluations generated by 50 repetitions of five-fold cross-validation. All three performance measures, i.e.,balanced accuracy, macro-F1, and Cohen's $\kappa$, were successfully obtained for every evaluation. For each dataset--method combination, 6,000 dependent held-out values were generated across the eight feature budgets and three classifiers. Accordingly, each dataset--method--metric summary represented all 24 budget--classifier conditions.

\Cref{tab:global} summarizes the aggregate performance across the five datasets. \ARISE\ achieved the highest equal-dataset mean and the best mean rank for each of the three evaluation metrics. SNR was the strongest comparator for balanced accuracy and macro-F1, with absolute differences in favour of \ARISE\ of 0.0221 and 0.0233, respectively. For Cohen's $\kappa$, MI was the closest comparator, with an absolute difference of 0.0284. When the 15 dataset--metric summaries were weighted equally, \ARISE\ achieved the highest overall three-metric mean (0.7650), followed by SNR (0.7400), MI (0.7377), Fisher (0.7370), Welch (0.7338), POS (0.6693), and Bhattacharyya (0.6632). Overall, these results indicate consistently stronger aggregate performance of \ARISE\ across the evaluated datasets and performance criteria, with SNR and MI representing the most competitive standalone alternatives.

	\begin{table}[h]
		\centering
		\begin{threeparttable}
\caption{Aggregate comparison across the five evaluated datasets. Values are equal-dataset means with mean ranks in parentheses; lower rank is better. The final column is the equal-weight mean across all 15 dataset--metric combinations. Bold and underline denote the best and second-best values, respectively.}

			\label{tab:global}
			\small
			\setlength{\tabcolsep}{5.0pt}
			\begin{tabular}{@{}lcccc@{}}
				\toprule
				Method & Balanced accuracy & Macro-F1 & $\kappa$ & Three-metric mean \\
				\midrule
				ARISE & \best{0.7932 (1.0)} & \best{0.7765 (1.0)} & \best{0.7252 (1.0)} & \best{0.7650} \\
				SNR & \second{0.7711 (2.6)} & \second{0.7532 (2.8)} & 0.6957 (3.0) & \second{0.7400} \\
				MI & 0.7680 (4.4) & 0.7483 (4.6) & \second{0.6968 (3.8)} & 0.7377 \\
				Fisher & 0.7683 (3.8) & 0.7504 (3.6) & 0.6924 (3.8) & 0.7370 \\
				Welch & 0.7658 (4.0) & 0.7474 (3.8) & 0.6884 (4.0) & 0.7338 \\
				POS & 0.7114 (6.8) & 0.6848 (6.8) & 0.6116 (6.8) & 0.6693 \\
				Bhattacharyya & 0.7076 (5.4) & 0.6773 (5.4) & 0.6046 (5.6) & 0.6632 \\
				\bottomrule
			\end{tabular}
		
		\end{threeparttable}
	\end{table}
	
	The violin plots in \cref{fig:ba-violin,fig:f1-violin,fig:kappa-violin} retain the full observed spread that the means suppress. Each dataset--method violin contains the same 6,000 outer-fold--budget--classifier values used in the corresponding mean, while its embedded box summarizes the median and interquartile range. The distributions overlap among methods within every dataset, showing that aggregate superiority does not imply uniform fold-level dominance. D2 exhibits the lowest central performance and substantial dispersion, whereas D3 and D4 generally occupy the strongest ranges. Because folds, budgets, and classifiers reuse observations and training partitions, violin width is descriptive density rather than sampling evidence.
	
\begin{figure}[H]
	\centering
	\includegraphics[width=0.8\linewidth]{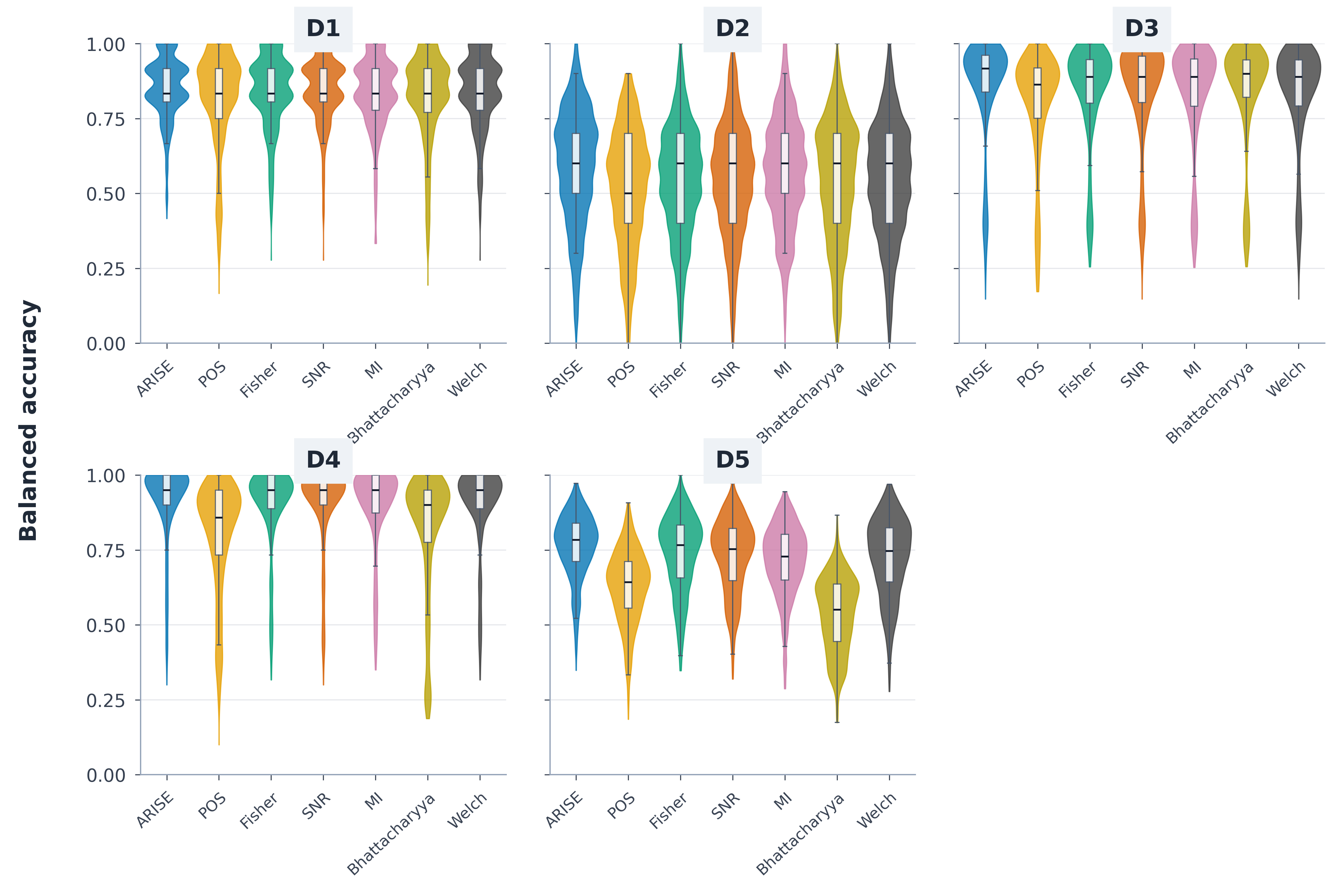}
	\caption{Balanced-accuracy distributions across the five datasets and seven methods.}
	\label{fig:ba-violin}
\end{figure}

\begin{figure}[H]
	\centering
	\includegraphics[width=0.8\linewidth]{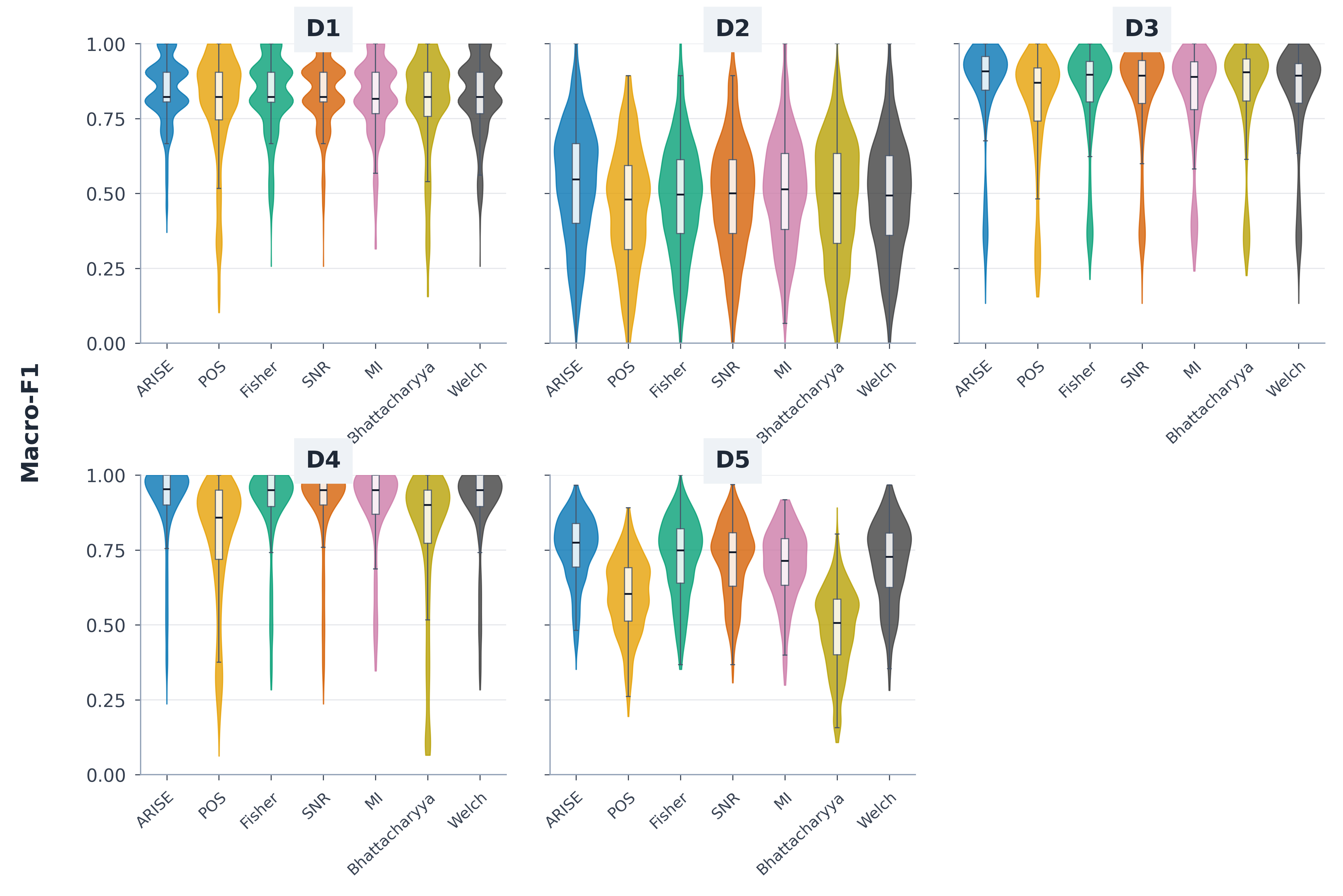}
	\caption{Macro-F1 distributions across the five datasets and seven methods.}
	\label{fig:f1-violin}
\end{figure}

\begin{figure}[H]
	\centering
	\includegraphics[width=0.8\linewidth]{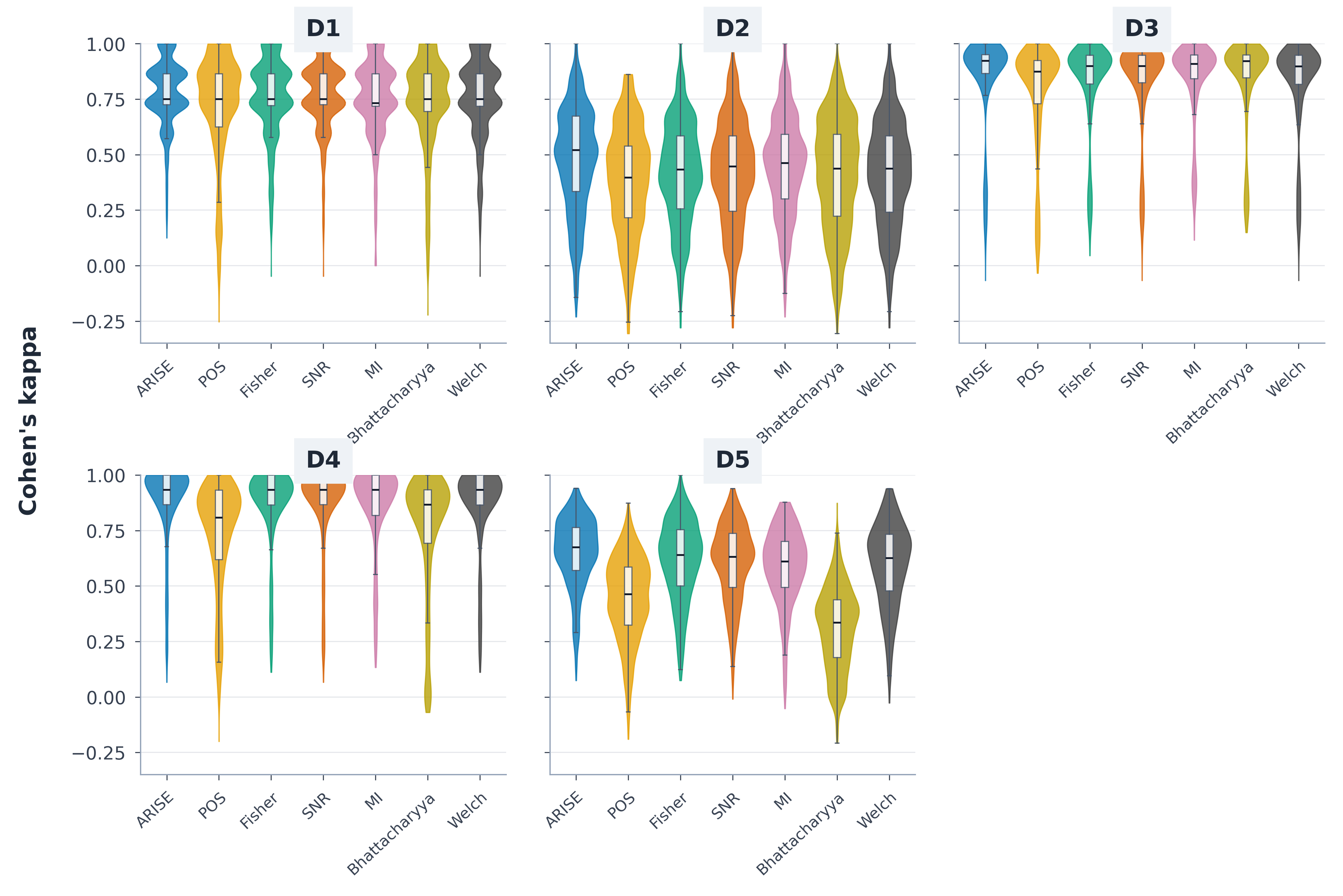}
	\caption{Cohen's $\kappa$ distributions across the five datasets and seven methods.}
	\label{fig:kappa-violin}
\end{figure}

	\subsection{Dataset-wise consistency and effect concentration}
	
	The aggregate position was consistent across datasets: \ARISE\ ranked first among the seven displayed methods in all 15 dataset--metric cells (\cref{tab:datawise}). The absolute advantage over the best remaining comparator ranged from 0.0051 to 0.0261 for balanced accuracy, from 0.0044 to 0.0316 for macro-F1, and from 0.0054 to 0.0383 for $\kappa$. D5 showed the largest advantage on every reported metric, whereas D1 had the smallest balanced-accuracy and macro-F1 margins and D3 the smallest $\kappa$ margin. 
\begin{table}[H]
	\centering
	\caption{Dataset-wise performance of ARISE. Values are means with ranks among the seven methods in parentheses, averaged across the 24 budget--classifier conditions.}
	\label{tab:datawise}
	\small
	\begin{tabular}{lccc}
		\toprule
		\textbf{Dataset} & \textbf{Balanced accuracy} & \textbf{Macro-F1} & \textbf{Cohen's $\kappa$} \\
		\midrule
		D1 & 0.8524 (1.0) & 0.8416 (1.0) & 0.7794 (1.0) \\
		D2 & 0.5903 (1.0) & 0.5351 (1.0) & 0.4744 (1.0) \\
		D3 & 0.8513 (1.0) & 0.8465 (1.0) & 0.8440 (1.0) \\
		D4 & 0.9083 (1.0) & 0.9045 (1.0) & 0.8760 (1.0) \\
		D5 & 0.7638 (1.0) & 0.7546 (1.0) & 0.6520 (1.0) \\
		\bottomrule
	\end{tabular}
\end{table}

\begin{table}[htbp]
	\centering
	\caption{ Friedman rank analysis across the five datasets.}
	\label{tab:statistical}
	\small
	\begin{tabular}{lccc}
		\hline
		Metric & Friedman statistic & df & $P$-value \\
		\hline
		Balanced accuracy & 22.457 & 6 & 0.00100 \\
		Macro-F1 & 22.286 & 6 & 0.00107 \\
		Cohen's $\kappa$ & 21.943 & 6 & 0.00124 \\
		\hline
	\end{tabular}
\end{table}

The Friedman analysis indicated differences in method rankings across all three evaluation metrics (Table~\ref{tab:statistical}). The corresponding asymptotic $P$-values were 0.00100 for balanced accuracy, 0.00107 for macro-F1, and 0.00124 for Cohen's $\kappa$. However, because only five datasets constitute independent experimental blocks, these results should be interpreted as exploratory rather than confirmatory. Accordingly, greater emphasis is placed on the consistency of dataset-wise rankings and the magnitude of the observed performance differences.

	\subsection{Feature-budget behavior and selected-set stability}
	
	The ARISE condition structure is shown compactly in \cref{fig:performance-atlas}: five dataset rows, three metric panels, and 24 classifier--budget columns formed by three classifiers and eight feature budgets. Thus, the atlas contains $5\times3\times3\times8=360$ condition means. Averaging those 360 condition means equally, the common-budget composite rose from 0.4595 at one feature to 0.8239 at 20, 0.8241 at 25, and a maximum of 0.8289 at 30, before decreasing to 0.8166 at 35. At 30 features, the separate means were 0.8501 for balanced accuracy, 0.8359 for macro-F1, and 0.8007 for $\kappa$. Thirty features is therefore the best observed \emph{common} budget for this panel, not a universal optimum: the dataset-specific composite maxima occurred at 20 features for D1, 15 for D2, 35 for D3 and D4, and 30 for D5.
	
\begin{figure}[H]
	\centering
	\includegraphics[width=0.9\linewidth]{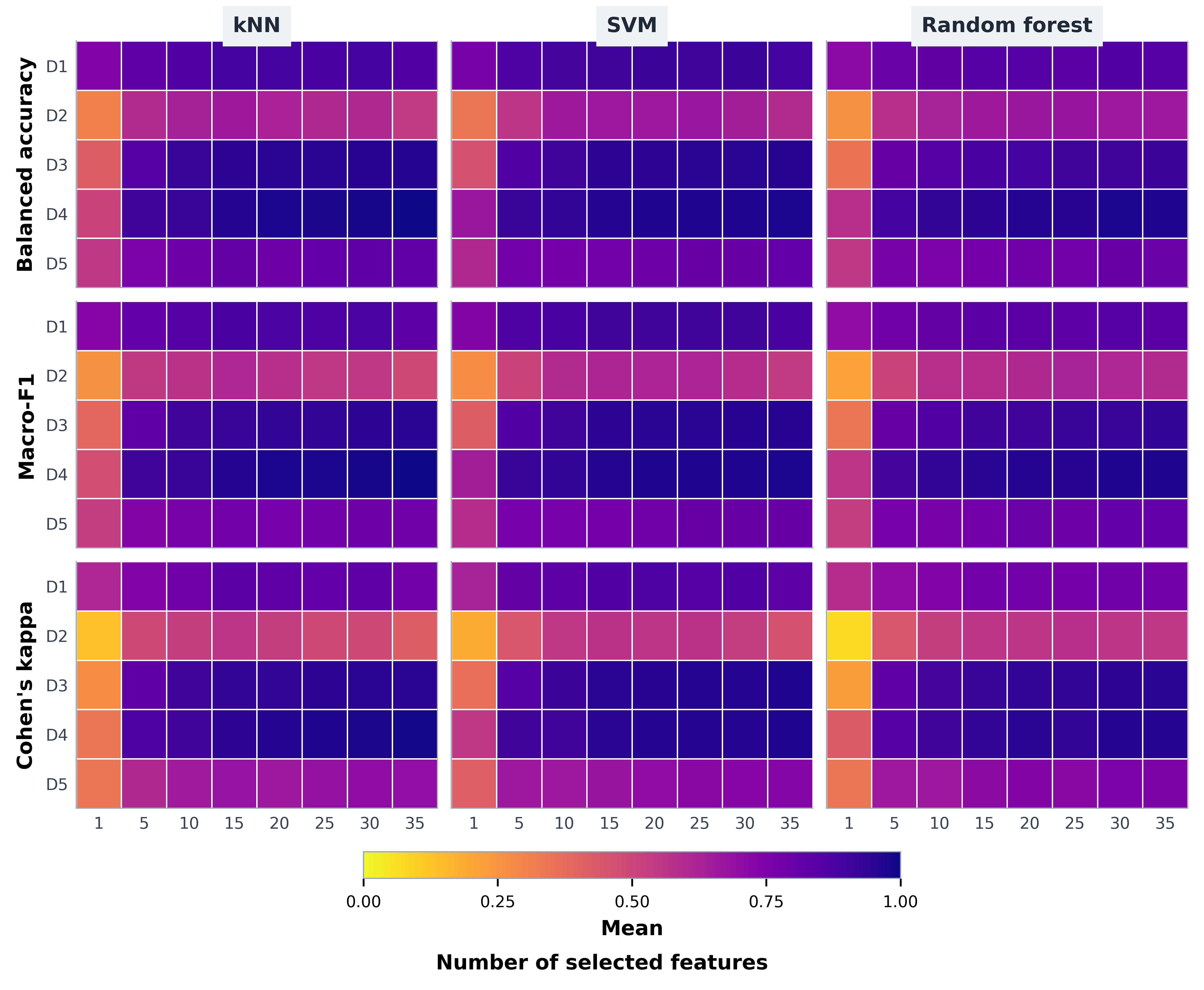}
	\caption{ARISE performance across datasets, classifiers, metrics, and feature budgets.}
	\label{fig:performance-atlas}
\end{figure}

	Predictive performance was accompanied by informative selected-set stability patterns (\cref{fig:stability}). At 35 features, the mean pairwise Jaccard overlap was 0.836 for D1, 0.457 for D2, 0.461 for D3, 0.337 for D4, and 0.237 for D5. The strong overlap for D1 reflects highly consistent feature recovery within its compact 45-feature universe, where selecting 35 features naturally increases agreement across folds. For the higher-dimensional datasets, the more moderate overlap suggests that \ARISE\ preserves predictive performance while allowing flexibility in the choice of near-equivalent feature subsets across resamples. This pattern is consistent with the presence of multiple competitive biomarker combinations in complex molecular settings, rather than reliance on a single rigid panel. Overall, the stability results support \ARISE\ as a robust feature-selection framework, while indicating that biomarker-level interpretation should be complemented by external validation and assay-level replication.
	
	\begin{figure}[H]
		\centering
		\includegraphics[width=0.98\linewidth]{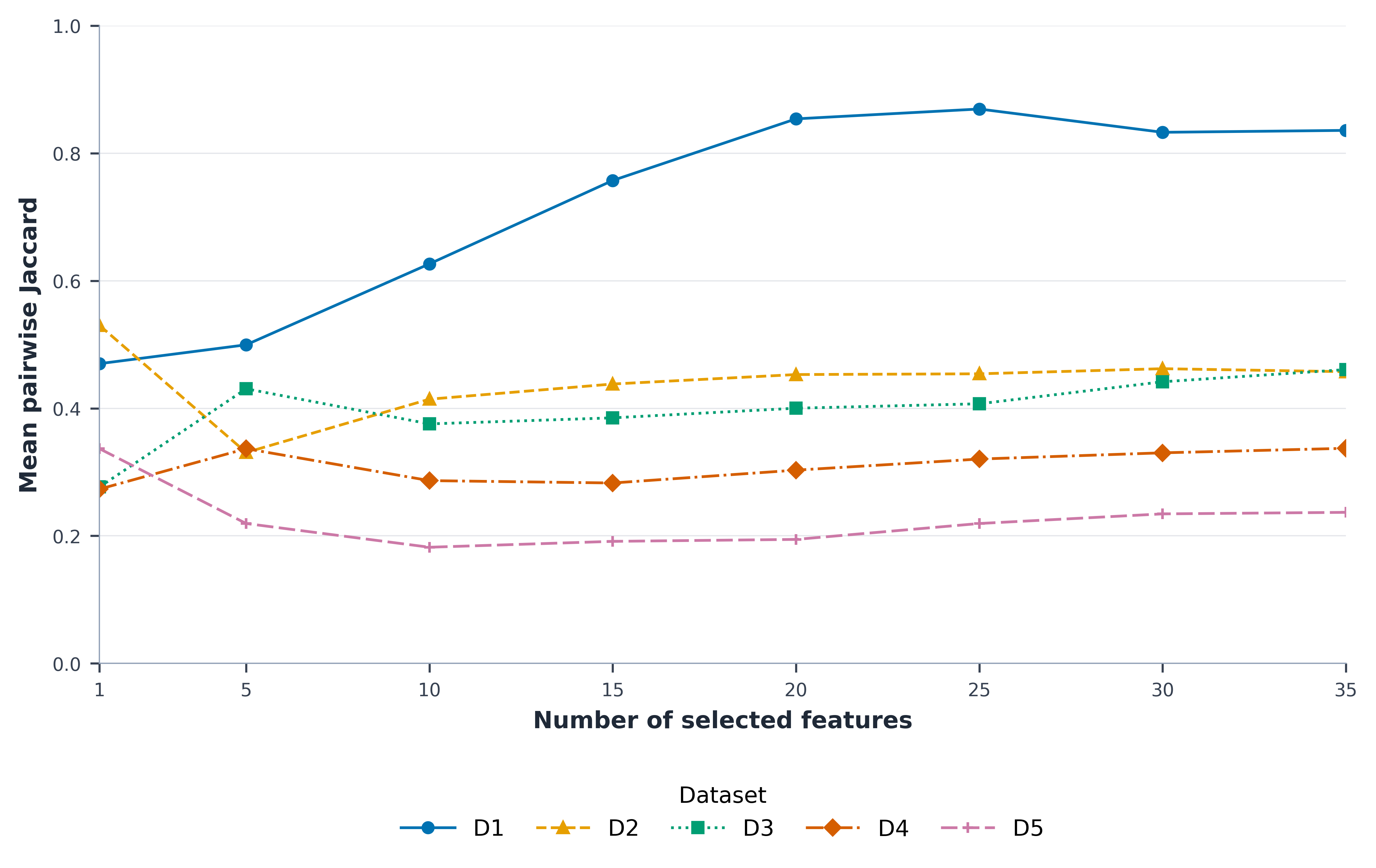}
		\caption{Selected-set stability of ARISE across feature budgets for D1--D5.}
		\label{fig:stability}
	\end{figure}

	\FloatBarrier
	\section{Discussion}
	\label{sec:discussion}
	
	Across the five datasets, three reported metrics, and three fixed classifiers, \ARISE\ had the largest equal-dataset mean and rank 1 in all 15 dataset--metric cells ($5$ datasets $\times 3$ metrics), and its three-metric composite exceeded those of SNR and MI, the second- and third-ranked standalone methods. The agreement among balanced accuracy, macro-F1, and $\kappa$ is useful because the measures stress different failure modes: minority-class recall, classwise precision--recall balance, and chance-corrected concordance. Their distributional plots nevertheless show substantial within-dataset overlap, indicating that aggregate superiority does not imply uniform fold-level dominance. The results therefore support the proposed architecture as a methodological candidate for further independent evaluation rather than as evidence of universal superiority.
	
	The method's value lies in how it reconciles signal geometries rather than in a new isolated score. Percentile normalization places heterogeneous evidence on a common scale; class-balanced subsampling reduces majority-class control of robustness; residual correlation discourages redundant within-class variation; pair coverage rewards features that resolve different class contrasts; and soft consensus averages parameters of eligible near-tied profiles instead of choosing a single winner. The formal results support this interpretation: the construction is bounded, the coverage term is monotone submodular, redundancy is invariant to class-specific shifts, paths are deterministically nested, and consensus parameters remain within the eligible-profile hull. These are properties of internal coherence and reproducibility, not guarantees of biological causality or out-of-distribution performance. A component-by-component ablation remains necessary to quantify which mechanisms generate the observed gains.
	
	The five datasets expose distinct translational contexts. D1 shows that \ARISE\ can exploit a compact methylation panel, but because only 45 CpGs enter the package object, it does not demonstrate de novo genome-scale assay discovery. D2 is a controlled Nutrimouse experiment and broadens multiclass signal geometry without constituting human clinical validation. D3 provides the largest sample and strong discrimination among acute lymphoblastic leukemia subtypes, yet remains internal resampling rather than cohort transfer. D4 is a useful four-class molecular stress test, but its processed native codes limit clinical interpretation. D5 provides a three-class peripheral-blood expression setting relevant to inflammatory bowel disease and also contributed the largest proposed-method margins. This heterogeneity supports method development while preventing the pooled mean from being read as the expected effect for a single clinical population.
	
	Thirty selected features was the best common setting, but the optimum ranged from 15 to 35 by dataset. That variation argues for budget choice inside the training loop when a study has no prespecified assay size; if laboratory cost fixes the panel size, the budget should instead be declared before evaluation. Stability provides a second qualification. D4 and D5 retained low overlap at larger budgets, and D1's high overlap is partly a combinatorial consequence of selecting most of a 45-feature universe. Accurate prediction and reproducible biomarker identification are therefore separate objectives. Future work should use chance-corrected stability, pathway coherence, assay replication, and untouched cohorts rather than interpreting selection frequency alone.
	
	Contemporary medical AI increasingly combines stability, explainability, networks, and multi-objective search \cite{Chereda2024Stable,Zhou2025Rough,Shi2025DRExplainer}. \ARISE\ is deliberately more transparent than graph or foundation models: it returns one transparent ordered panel, separates selector adaptation from fixed-learner comparison, and exposes each source of gain. That transparency is attractive for small studies, while pathway structure, assay uncertainty, patient covariates, calibration, treatment decisions, and subgroup robustness remain outside the current model. A clinically credible extension should freeze the selector before testing, compare it with stable and network-based methods on untouched cohorts, and quantify calibration, decision-curve utility, assay feasibility, and the reproducibility of individual selected features.
	
	\section{Conclusion}
	
	\ARISE\ provides a coherent multiclass feature-selection framework combining complementary relevance evidence, balanced-subsample robustness, within-class residual redundancy control, class-pair coverage, nested profile racing, and convex consensus. Across five evaluated datasets, three reported metrics, and three fixed classifiers, it ranked first in all 15 dataset--metric cells. Although 30 features produced the best common-budget mean, the dataset-specific optima differed, highlighting the value of evaluating the method across multiple feature-set sizes rather than relying on a single prespecified panel size. The consistent results across balanced accuracy, macro-F1, and $\kappa$ further indicate that the observed advantage is not confined to a single performance criterion.
	
	The study also has several clearly defined methodological boundaries that help delimit the current evidence. The use of five heterogeneous benchmark datasets provides variation in sample size, dimensionality, molecular platform, and class structure, while preserving a controlled experimental setting for evaluating the selector. Package-level preprocessing was retained to ensure reproducibility with established public data objects, whereas all additional preprocessing, feature selection, tuning, and model fitting introduced in this study were restricted to the corresponding training partitions. The inclusion of a preclinical dataset broadens the range of signal structures examined, although clinical interpretation remains focused on methodological performance rather than direct translational claims. Similarly, the use of fixed classifiers was intentional, allowing differences in predictive performance to be attributed more directly to feature selection rather than to extensive learner-specific hyperparameter optimization. The repeated five-fold design provided extensive resampling across 250 held-out folds per condition, while the dataset, rather than the individual fold, remained the unit of across-study inference. Finally, variation in selected-set overlap across datasets highlights the ability of \ARISE\ to identify alternative competitive feature subsets in high-dimensional settings and also motivates complementary stability analyses when the goal extends from prediction to biomarker interpretation. These design choices define a rigorous benchmarking stage and provide a clear basis for future evaluation on prospectively specified datasets, independent external cohorts, assay-level replication, calibration, decision-curve analysis, subgroup assessment, and chance-corrected feature stability before clinical translation.

	\section*{Data availability}
	The datasets analyzed in this study are publicly available through the data resources and originating studies identified in \cref{tab:datasets,sec:provenance}. The analyzed object names and package versions are reported in the Methods section to support object-level provenance.
	
	\section*{Declaration of competing interest}
	The authors declare that they have no known competing financial interests or personal relationships that could have appeared to influence the work reported in this paper.
	
	\section*{Funding}
This work was sponsored by the United Arab Emirates University under grand 12B086.
	
\section*{CRediT authorship contribution statement}

\textbf{Zardad Khan:} Conceptualization, Methodology, Software, Formal analysis, Visualization, Supervision, Writing -- original draft, Writing -- review \& editing. \textbf{Amjad Ali:} Methodology, Software, Investigation, Data curation, Validation, Formal analysis, Writing -- review \& editing. \textbf{Naz Gul:} Investigation, Data curation, Formal analysis, Validation, Visualization, Methodology, Writing -- review \& editing. \textbf{Sheema Gul:} Investigation, Data curation, Resources, Validation, Visualization, Writing -- review \& editing. \textbf{Saeed Aldahmani:} Conceptualization, Resources, Funding acquisition, Project administration, Supervision, Validation, Writing -- review \& editing.

	\bibliographystyle{elsarticle-num}
	\bibliography{references}
	
\end{document}